\documentclass[11pt]{article}

\usepackage{graphicx}
\usepackage{amsfonts,amssymb,amsmath,amsthm}
\usepackage[numbers,square]{natbib}
\usepackage{algorithm,algorithmic}
\usepackage{hyperref}
\usepackage{multirow}
\hypersetup{backref,colorlinks=true,citecolor=blue,linkcolor=blue}
\usepackage[left=1in,top=1in,right=1in,bottom=1in,nohead]{geometry}
\newtheorem{theorem}{Theorem}[section]
\newtheorem{definition}[theorem]{Definition}
\newtheorem{lemma}[theorem]{Lemma}

\newtheorem{corollary}[theorem]{Corollary}

\newcommand{\given}{\mbox{ }\vert\mbox{ }}
\newcommand{\F}{\mathcal{F}}
\newcommand{\G}{\mathcal{G}}
\newcommand{\E}{\mathbb{E}}
\renewcommand{\P}{\mathbb{P}}
\newcommand{\R}{\mathbb{R}}

\newcommand{\email}[1]{\href{mailto:#1}{#1}}
\renewcommand{\eqref}[1]{(\ref{eq:#1})}
\DeclareMathOperator*{\argmin}{argmin}
\DeclareMathOperator*{\vcd}{VCD}
\newcommand{\vnorm}[1]{\left|\left| #1 \right|\right|}

\renewcommand{\hat}[1]{\widehat{#1}}

\title{Estimated VC dimension for risk bounds}
\author{Daniel J. McDonald\\Carnegie Mellon
  University\\\email{danielmc@cmu.edu} \and Cosma Rohilla
  Shalizi\\Carnegie Mellon University\\\email{cshalizi@cmu.edu}
  \and Mark Schervish\\Carnegie Mellon
  University\\\email{mark@cmu.edu}}
\date{Version: \today}

\begin{document}

\maketitle

\begin{abstract}
  Vapnik-Chervonenkis (VC) dimension is a fundamental measure of the
  generalization capacity of learning algorithms.  However, apart from a few
  special cases, it is hard or impossible to calculate
  analytically. \citet{VapnikLevin1994} proposed a technique for estimating the
  VC dimension empirically.  While their approach behaves well in simulations,
  it could not be used to bound the generalization risk of classifiers, because
  there were no bounds for the estimation error of the VC dimension itself. We
  rectify this omission, providing high probability concentration results for
  the proposed estimator and deriving corresponding generalization bounds.
\end{abstract}

\section{Introduction}
\label{sec:introduction}

Statistical learning theory is fundamentally concerned with picking, out of
some class of plausible or convenient models, ones whose predictions will be
nearly optimal.  Statistical optimality is most often demonstrated by
controlling the risk, or generalization error, of predictive models, i.e.,
their expected inaccuracy on new data from the same source as that used to fit
the model.  The paradigmatic case confronts the learner with a labeled set of
training examples $Z=\{(y_1,x_1),\ldots,(y_n,x_n)\}$ drawn independently from a
distribution $\mu$ over $\mathcal{Y} \times \mathcal{X}$. For concreteness, we
take the standard task of pattern recognition with vector features, setting
$\mathcal{Y}=\{0,1\}$ and $\mathcal{X}=\R^p$.  Our contribution is to
controlling the risk of pattern recognition when using analytically intractable
models.

Consider a class $\F$ of possible predictors, that is a collection of functions
from $\mathcal{X}$ to $\mathcal{Y}$. From this class, the learner uses the
training set to choose some $f \in \F$, hoping to make as few errors in the
future as possible when facing similar data. This amounts to controlling the
\emph{risk} of $f$
\begin{equation}
  \label{eq:20}
  R_n(f) = \E_\mu[ I(Y \neq f(X))],
\end{equation}
where $I(A)$ is the indicator of the event $A$.  Since the distribution $\mu$
is unknown, the risk cannot be calculated explicitly, so learners often proxy
it by the \emph{empirical risk} of $f$,
\begin{equation}
  \label{eq:19}
  \widehat{R}_n(f,Z) = \frac{1}{n} \sum_{i=1}^n I(Y_i \neq f(X_i)),
\end{equation}
which we will abbreviate $\hat{R}_n(f)$ when possible.  Since \eqref{19}
approximates \eqref{20}, we can choose a good predictor $\hat{f}$ by solving
\begin{equation*}
  \label{eq:21}
  \hat{f} = \argmin_{f \in \F} \frac{1}{n} \sum_{i=1}^n I(Y_i \neq f(X_i)).
\end{equation*}
This process is \emph{empirical risk minimization}, or ERM.  ERM itself is
quite general, and with appropriate loss functions includes ordinary least
squares regression, maximum likelihood, nonparametric density estimation, and
$M$-estimation.

The next step in the statistical learning paradigm is to evaluate the
performance of ERM. Is $\hat{f}$ consistent (in risk) for $f$? What is the rate
of convergence? Can we control the generalization error of the chosen
$\hat{f}$?  In fact, all of these questions are
answered. \citet{VapnikChervonenkis1991} gave necessary and sufficient
conditions for uniform convergence of $\hat{R}_n(f)$ to $R_n(f)$ in terms of
the VC entropy.  However, the VC entropy itself depends on the unknown
distribution $\mu$.  To get around this, we look instead at a bound for the VC
entropy which is uniform over probability measures: the \emph{growth function},
which can be calculated from the VC dimension, which is based on the shattering
coefficient.

\begin{definition}
  Let $\mathbb{U}$ be some (infinite) set and let $\mathcal{S}$ be a finite
  subset of $\mathbb{U}$. Let $\mathcal{C}$ be a family of subsets of
  $\mathbb{U}$. We say that $\mathcal{C}$ shatters $S$ if for every $S^{\prime}
  \subseteq S$, $\exists C \in\mathcal{C}$ such that $S^{\prime} = S \cap C$.
\end{definition}

\begin{definition}
  [VC dimension] The \emph{Vapnik-Chervonenkis (VC) dimension} of $\mathcal{C}$
  is
  \begin{equation*}
    \vcd(\mathcal{C}) := \sup \{ \mbox{card } S : S\mbox{ is shattered by
    } \mathcal{C} \}.
  \end{equation*}
\end{definition}

Application of VC dimension to classes of functions is reasonably
straightforward for pattern recognition. To $f\in\F$, associate the set $C_f =
\{ u \in \mathbb{U} : f(u)=1\}$, and associate to $\F$ the class
$\mathcal{C}_\F := \{ C_f : f \in \F\}$. Then define $\vcd{(\F)} :=
\vcd{(\mathcal{C}_\F)}$.

VC dimension is just one of many ways to measure the richness or complexity of
a class of functions.  Others include covering numbers,
Pseudo-dimension~\citep{Pollard1984}, fat-shattering dimension, and Rademacher
complexity~\citep{BartlettMendelson2002}. Heuristically, larger complexity
leads to smaller minimum risk but higher estimation variance, and thus it is
important to balance the complexity of the function class with the amount of
data available.  For VC dimension, \citet{Vapnik2000} shows that a sufficient
condition for uniform risk consistency is that
\begin{equation*}
  \label{eq:22}
  \lim_{n\rightarrow\infty} \frac{\log{GF(h^*,n)}}{n} = 0,
\end{equation*}
where $\log{GF(h,n)}\leq h(\log{(n/h)+1})$ is the growth function and $h^*=\vcd
(\F)$ is the VC dimension of the function class. Furthermore,
\citet{Vapnik1998,Vapnik2000} proves a concentration result of the empirical
risk around the true risk: for any $\rho>0$
\begin{equation}
  \label{eq:23}
    \P_\mu \left( \sup_{f \in \F} \left|R_n(f) - \hat{R}_n(f)\right| >
      \rho\right) < 4 GF(h^*,2n) \exp\left\{-n\rho^2\right\}.
\end{equation}
Similar bounds exist for other loss functions such as margin loss, loss
functions constrained to a compact interval, or extended real-valued loss
functions for regression problems.

Given a function class $\F$, knowing $h^*=\vcd (\F)$ is crucial to using these
sorts of results. However, for many interesting function classes (support
vector machines, multi-layer neural networks, random forests, etc.) this
knowledge is entirely unavailable. The combinatorial nature of VC dimension
makes it very difficult to find in interesting cases. As a remedy,
~\citet{VapnikLevin1994} propose a way to estimate the VC dimension by
simulation. While the authors showed its accuracy by estimating the VC
dimension of linear classifiers (known to be the number of covariates with an
extra degree of freedom for the intercept), estimated VC dimension cannot be
simply plugged in to finite-sample concentration results (such as \eqref{23}),
because the estimates themselves fluctuate around the true values.  Since VC
dimension is only useful to the extent it lets us bound generalization risk,
this presents a problem.  In this paper, we rectify this situation.

We prove two main results.  First, we show that, using the procedure of
~\citep{VapnikLevin1994}, the estimated VC dimension, $\hat{h}$, will
concentrate around the truth, $h^*$, with high probability:
\begin{theorem}
  \label{thm:main}
  Let $\delta>\frac{4}{\sqrt{2mk}}\max\{24c_1,29\}$ and suppose that $h^* \leq
  M$. Then
  \begin{equation*}
    \label{eq:1a}
    \P\left( |\widehat{h}-h^*| > \delta\right) \leq 13\exp \left\{
      -\frac{mkc_2\delta^2}{16c_3} \right\}
  \end{equation*}
  where $c_1$, $c_2$, and $c_3$ are constants given in the proof and in
  Table~\ref{tab:constants}, and $k$ and $m$ are integers freely chosen as part
  of the simulation procedure.
\end{theorem}

Second, we show that if we use the estimated VC dimension, we can still recover
bounds like that in \eqref{23}:
\begin{theorem}
  \label{thm:bound}
  Choose $\delta$ as in Theorem~\ref{thm:main}. Let $\rho>0$. Set
  $$\varphi=13\exp \left\{-\frac{mkc_2\delta^2} {16c_3}\right\}.$$  Then, for any
  classifier $f \in \F$ where $\F$ has estimated VC dimension $\widehat{h}$, we
  have
  \begin{equation}
    \label{eq:18}
      \P(\sup_{f\in\F}\left|R_n(f) - \widehat{R}_n(f)\right|>\rho)
      \leq 4GF(\hat{h}+\delta,2n) \exp\{-n\rho^2\}(1-\varphi) +\varphi.
  \end{equation}
\end{theorem}
The first term on the right of \eqref{18} is the same as the original bound in
\eqref{23}, except that the true VC dimension is replaced with its estimate
$\hat{h}$ plus a small fudge factor $\delta$. The second term depends on the
confidence that we have in our estimate, through $\varphi$.  The estimation
procedure allows us to estimate $h^*$ arbitrarily well, given infinite
computational time, through the choice of $m$ and $k$.  Of course this is
infeasible in practice, but Theorem~\ref{thm:bound} allows the user to trade
computational time for statistical accuracy.

The remainder of this paper provides details for the proofs of our two main
theorems. Section~\ref{sec:estimation} summarizes the estimation procedure
developed in~\citet{VapnikLevin1994}. Section~\ref{sec:proof-results} proves
both theorems, drawing on empirical process theory. Because there is a lot of
notation, we summarize it in Table~\ref{tab:constants}. Finally,
Section~\ref{sec:discussion} concludes and provides some ideas for future work.
\begin{table*}[t!]
  \centering
  \caption{Constants and important notation}  
  \begin{tabular}{|cl|}
    \hline\hline
    Notation & Meaning\\
    \hline
    $h^*$ & the VC dimension of the function class $\F$\\
    $\hat{h}$ & the estimate of VC dimension via \eqref{10}\\
    $M$ & we assume $h^* \leq M$\\
    $\Phi_h(n)$& $\begin{cases} 1 & n < h/2\\ a\frac{\log\frac{2n}{h}+1}{\frac{n}{h} -
        a''} \left( \sqrt{ 1+
          \frac{a'\left(\frac{n}{h}-a''\right)}{\log\frac{2n}{h}+1}} +
        1\right) & \mbox{else}.
    \end{cases}$\\
    $a$& 0.16\\
    $a'$& 1.2\\
    $a''$& 0.14927\\
    $\varphi$ & $\displaystyle 13\exp\left\{-\frac{mkc_2\delta^2}{16c_3}\right\}$\\
    $GF(h,n)$ & $\leq h(\log(n/h) +1)$\\
    $c(n,M)$ & \multirow{2}{*}{$\begin{cases} \mbox{Lipschitz-like constants such that $\forall n$:} \\ c(n,M)|h-h'| \leq
        |\Phi_h(n)-\Phi_{h'}(n)| \leq L(n)|h-h'|\end{cases}$}\\
    $L(n)$ & \\
    $N(\eta,\G)$ &the $\eta$-covering number of $\G$\\
    $H(\eta,\G)$ &the $\eta$-entropy of $\G$\\
    $k$, $m$ & integers chosen for the simulation in Algorithm~\ref{alg:gen}\\
    $c_1$ &$\displaystyle (c'+1/4)\sqrt{\log (4c'+1)} -
    \frac{\sqrt{\pi}}{8}\mbox{erfi}(\sqrt{4c'+1} )$\\
    $c'$ & $\displaystyle \frac{1}{k} \sum_{\ell=1}^k L^2(n_\ell)$\\
    $c_2$ & $\displaystyle \frac{1}{k} \sum_{\ell=1}^k c^2(n_\ell,M)$\\
    $c_3$ & 2304\\
    \hline\hline
  \end{tabular}
  \label{tab:constants}
\end{table*}

\section{Estimation}
\label{sec:estimation}

\citet{VapnikLevin1994} show that the expected maximum deviation between the
empirical risks of a classifier on two datasets can be bounded by a function
which depends only on the VC dimension of the classifier. In other words, given
a collection of classifiers $\F$, and two data sets
$W=\{(y_1,x_1),\ldots,(y_n,x_n)\}$ and $W' =
\{(y'_1,x'_1),\ldots,(y'_n,x'_n)\}$, we have the bound
\begin{equation}
  \label{eq:1}
    \xi(n) :=\E\left[ \sup_{f\in\F} (\widehat{R}_n(f,W) -
      \widehat{R}_n(f,W')) \right]\leq \begin{cases} 1 & n/h^* \leq
      \frac{1}{2}\\ C_1 \frac{ \log(2n/h^*) + 1}{n/h^*} & \mbox{if
        $n/h^*$ is small}\\ C_2\sqrt{ \frac{ \log(2n/h^*) + 1}{n/h^*}} &
      \mbox{if $n/h^*$ is large}. \end{cases}
\end{equation}
We can bound \eqref{1} by $\Phi_{h^*}(n)$, viewed as a function of $n$ and
parametrized by $h$:
\begin{equation}
  \label{eq:2}
  \Phi_h(n) = \begin{cases} 1 & n < h/2\\ a\frac{\log\frac{2n}{h}+1}{\frac{n}{h} -
      a''} \left( \sqrt{ 1+
        \frac{a'\left(\frac{n}{h}-a''\right)}{\log\frac{2n}{h}+1}} +
      1\right) & \mbox{else}. \end{cases}
  \end{equation}
Here the constants $a=0.16$, $a'=1.2$ were determined numerically
in~\citep{VapnikLevin1994} to adjust the trade-off between ``small'' and
``large'' in \eqref{1}, and $a''=0.14927$ was chosen so that $\Phi(0.5)=1$
(this choice depends only on $a$ and $a''$).  Furthermore, the bound is tight.
Since~\eqref{2} is known up to $h$, we can estimate it given knowledge of the
maximum deviation on the left side of~\eqref{1}. Of course, we do not have such
knowledge, but we can generate observations
\begin{equation*}
  \label{eq:3}
  \widehat{\xi}(n) = \Phi_h(n) + \epsilon(n)
\end{equation*}
at design points $n$. Here $\epsilon$ is mean zero noise (since the bound is
tight) having an unknown distribution with support on $[0,1]$. Given enough
such observations at different design points $n_\ell$, we can then estimate the
true VC dimension $h^*$ using nonlinear least squares. Of course, generating
$\widehat{\xi}(n_\ell)$ is nontrivial. \citet{VapnikLevin1994} give an
algorithm for generating the appropriate observations. Essentially, at each
(fixed) design point $n_\ell: \ell \in \{1,\ldots,k\}$, we simulate $m$ data
points $(\widehat{\xi}_i(n_\ell),\Phi_h(n_\ell))$, for $i=1,\ldots,m$, so as to
approximate $\xi(n_\ell)$ as defined in (\ref{eq:1}). This procedure is shown
in Algorithm~\ref{alg:gen}.
\begin{algorithm*}[t]
  \caption{Generate $\widehat{\xi}(n_\ell)$}
  Given a collection of possible classifiers $\F$ and a grid of design points
  $n_1,\ldots,n_k$, generate $\widehat{\xi}(n_\ell)$. Repeat the procedure at
  each design point, $n_\ell$, $m$ times.
  \begin{algorithmic}[1]
    \STATE Generate a data set from the same sample space
    $\mathcal{Y}\times\mathcal{X}$ as the training sample that is
    independent of the training sample. The generated set should be of
    size $2n_\ell$: $\{(y_1,x_1),\ldots,(y_{2n_\ell},x_{2n_\ell})\}$. 
    \STATE Split the data set into two equal sets, $W$ and $W'$.
    \STATE Flip the labels ($y$ values) of $W'$.
    \STATE Merge the two sets and train the classifier simultaneously on the entire
    set: $W$ with the ``correct'' labels and $W'$ with the ``wrong'' labels.
    \STATE Calculate the training error of the estimated classifier
    $\widehat{f}$ on $W$ with the `correct' labels and on $W'$ using the
    ``correct'' labels.
    \STATE Set $\widehat{\xi_i}(n_\ell)=|\widehat{R}_{n_\ell}(\hat{f},W) - \widehat{R}_{n_\ell}(\hat{f},W')|$.
    \STATE Set $\widehat{\xi}(n_\ell)=\frac{1}{m} \sum_{i=1}^m \widehat{\xi_i}(n_\ell)$.
  \end{algorithmic} 
  \label{alg:gen}
\end{algorithm*}
\citet{VapnikLevin1994} show that this algorithm works well in practice,
recovering the known VC dimension of linear classifiers ($p+1$ for $p$
explanatory variables and an intercept) and demonstrating that the method for
generating the dataset does not affect the algorithm's
performance.\footnote{There are of course ways to generate data in so
  that this procedure will fail, e.g., generating the data with too-regular
  determinism, or with dependence.  We refer the cautious reader to
  \citet{VapnikLevin1994}. We also return to this point at the end of
  \S\ref{sec:proof-results}.}  In the next section, we prove our main
result, 
showing that in fact, the estimate concentrates around the truth with high
probability.

\section{Proof of results}
\label{sec:proof-results}

We now prove Theorem~\ref{thm:main} and Theorem~\ref{thm:bound}. The proofs
draw heavily on the empirical process techniques of~\citet{Geer1990}
and~\citet{Geer2000}; however, those works ignored constants, and made stronger
assumptions than necessary for the case at hand.  We strive to make our results
as self-contained as possible, appealing to~\citep{Geer2000} only for the proof
of Corollary~\ref{cor}.

Our goal is to show that the estimated VC dimension $\hat{h}$ is close to the
true dimension $h^*$.  This will mean showing that $\Phi_{\widehat{h}}$ is close
to $\Phi_{h^*}$ when averaged over the design points $n_{\ell}$.  It will
be convenient to introduce a norm and inner product for functions
$g:\R\rightarrow\R$:
\begin{align*}
  \label{eq:13a}
  \vnorm{g}_k^2 &= \frac{1}{k}\sum_{\ell=1}^k{g(n_\ell)^2}\\
  (\epsilon,g)_k &= \frac{1}{k}\sum_{\ell=1}^k{\epsilon(n_\ell) g(n_\ell)}.
\end{align*}
So we take as our estimate of $h^*$
\begin{equation*}
  \label{eq:10}
  \widehat{h} = \argmin_{h \in [0,M]}{ \vnorm{
  \widehat{\xi} - \Phi_h}_k},
\end{equation*}
and our immediate goal is control over $\vnorm{\Phi_{\widehat{h}} -
  \Phi_{h^*}}^2_k$.

For every $f\in\F$ and every dataset $Z$, $\hat{R}_n(f,W)$ is bounded between 0
and 1. Therefore, the residuals $\epsilon(n_\ell)$ are also in $[0,1]$. In
fact, we can show that they are subgaussian.
\begin{lemma} At all design points $n$,
  \label{lem:1}
  \begin{equation}
    \E[\exp\{t\epsilon(n)\}] \leq \exp\{t^2/8m\}.\label{eq:5}
  \end{equation}
\end{lemma}
\begin{proof}
  By a standard Hoeffding type argument, we have that
  \begin{equation*}
    \label{eq:6}
    \E[\exp\{t\epsilon_i(n)\}] \leq \exp\{t^2/8\}.
  \end{equation*}
  Therefore
  \begin{align*}
   \E[\exp\{t\epsilon(n)\}] &= \E\left[\prod_{i=1}^m\exp\{t\epsilon_i(n)/n\}\right]\\
   &\leq \exp\left\{\frac{mt^2}{m^28}\right\} \\
   &=\exp\{t^2/8m\}.
  \end{align*}
\end{proof}
The next step is to show that we can control weighted averages of the
$\epsilon(n_\ell)$.
\begin{lemma}
  \label{lem:2}
  Suppose $\epsilon_\ell :=\epsilon(n_\ell)$ are random variables
  satisfying~\eqref{5}. Then for any $\gamma \in \R^k$ and $\rho>0$,
  \begin{equation*}
    \label{eq:4a}
    \P \left( \left|\sum_{\ell=1}^k {\epsilon_\ell \gamma_\ell}\right|  > \rho\right)
    \leq 2\exp \left\{ - \frac{2m\rho^2} { \sum_{\ell=1}^k {\gamma_\ell^2}} \right\}.
  \end{equation*}
\end{lemma}
\begin{proof}
  Using a Chernoff bound, we have, for $t>0$
  \begin{equation*}
    \label{eq:5a}
    \P \left(  \sum_{\ell=1}^k{\epsilon_\ell \gamma_\ell}  >
      \rho\right) \leq \exp\left\{ -t\rho + \sum_{\ell=1}^k{\frac{\gamma_\ell t^2}{8m}} \right\}.
  \end{equation*}
  Taking
  \begin{equation*}
    \label{eq:6a}
    t = \frac{4m\rho}{\sum_{\ell=1}^k{\gamma_\ell}}
  \end{equation*}
  minimizes the right hand side. The same argument applies for
  $-\sum_{\ell=1}^k{\epsilon_\ell \gamma_\ell}$, so a union bound gives the
  result.
\end{proof}

In order to state our result about $\vnorm{\Phi_{\widehat{h}} - \Phi_{h^*}}_k$,
we must specify the complexity of the function class $\mathcal{G} :=\{ \Phi_h :
0 \leq h \leq M\}$, which we will measure with its \emph{entropy}.
\begin{definition}
  The functions $g_1, \ldots g_n$ are an {\em $\eta$-cover} of $\mathcal{G}$ if
  every $g \in \mathcal{G}$ is within $\eta$ of some $g_j$, $\|g-g_j\|_k \leq
  \eta$.  The {\em $\eta$-covering number} $N(\eta,\mathcal{G})$ is the
  cardinality of the smallest $\eta$-cover (or $\infty$ if there isn't one).
  The {\em $\eta$-entropy} is the log of the covering number,
  $H(\eta,\mathcal{G}) = \log N(\eta,\mathcal{G})$.
\end{definition}
While it may seem excessive to use covering numbers and entropy to deal with a
function class parametrized by a scalar, doing so lets us get much tighter
bounds than would otherwise be possible. The key to our argument will be the
entropy of the restricted class $\mathcal{G}(\tau) := \{\Phi_h \in \mathcal{G}
: \vnorm{\Phi-\Phi_{h^*}}_k \leq \tau\}$.
\begin{lemma}
  \label{lem:covnum}
  \begin{equation*}
    H(\eta,\mathcal{G}(\tau))
    \leq \log\left(\frac{4\tau/c'+\eta}{\eta}\right),\label{eq:11}
  \end{equation*}
  where $c'$ is defined below.
\end{lemma}
\begin{proof}
  $\Phi_h$ is bounded and differentiable in $h$ and therefore Lipschitz with
  constants $L(n)$. Thus
  \begin{equation*}
    \label{eq:13}
    \vnorm{\Phi_h - \Phi_{h'}}_k \leq \frac{1}{k}\sum_{\ell=1}^k
    L^2(n_\ell) |h-h'|.
  \end{equation*}
  Set $c'=\frac{1}{k}\sum_{\ell=1}^k{L^2(n_\ell)}$.  Covering a $\tau$ ball
  around $\Phi_h$ in the $\vnorm{\cdot}_k$ metric is then equivalent to
  covering a $\tau/c'$ ball around $h$ in the Euclidean metric. It is well
  known (cf.~\citep{Geer2000}) that
  \begin{equation*}
    \label{eq:14}
    H(\eta, B(\tau/c')) \leq \log\left(\frac{4\tau/c' + \eta}{\eta}\right).
  \end{equation*}
\end{proof}
The remaining proofs rely on the \emph{peeling device}. Intuitively, the idea
is that considering the entropy of larger and larger balls centered around
$\Phi_{h^*}$ will allow us to ``peel'' off sets of increasingly smaller
probability. This peeling argument is critical to our proof that
$\vnorm{\Phi_{\widehat{h}}- \Phi_{h^*}}_k$ is small with high probability.

To use peeling here, define $d(h) := \vnorm{\Phi_h - \Phi_{h^*}}_k$ and
consider a strictly increasing sequence $v_s$, starting with $v_0=0$ but
growing to $\infty$. We can peel $\G$ into $\G = \bigcup_{s=1}^\infty{\G_s},$
where $$\G_s = \{\Phi_h \in \G : v_{s-1} \leq d(h) < v_s \}.$$ Then we have
that for any $\rho>0$, and our residuals $\epsilon$ (which implicitly depend on
the choice of $g := \Phi_h$),
\begin{align*}
  \P\left(\sup_{g \in \G} \frac{|\epsilon|}{d(h)} >
      \rho\right)
  \leq \sum_{s=1}^\infty \P \left( \sup_{g \in \G_s}
    \frac{|\epsilon|}{d(h)} > \rho \right)
  \leq \sum_{s=1}^\infty \P \left( \sup_{g \in \G, d(h)< v_s} |\epsilon| >
    \rho v_{s-1}\right).
\end{align*}
This lets us get probability inequalities for the weighted process from
probabilities for the original process. We will want to allow the weights
$\gamma_\ell$ in Lemma~\ref{lem:2} to depend on functions $\Phi_h$, and in
particular investigate the behavior of the worst-case $h$. Taking $v_s=2^s$ for
$s>0$ and $v_0=0$ will allow us to derive an important corollary to
Lemma~\ref{lem:2} as well as Theorem~\ref{thm:emp}.

Choosing $v_s$ this way means that it is not enough to control the covering
number of the entire function class $\G$, but rather we must cover a sequence
of restricted classes $\G(\tau)$ with smaller and smaller balls. Therefore, we
will need the entropy sum,
\begin{equation*}
  \label{eq:24}
  J(\tau) := \sum_{s=1}^\infty{2^{-s}\tau \sqrt{ H(2^{-s}\tau, \G(\tau))}}.
\end{equation*}
which is bounded by the entropy integral,
\begin{equation*}
  \label{eq:25}
  J(\tau) \leq 2\int_0^\tau{du\sqrt{H(u,\mathcal{G}(\tau)}}
\end{equation*}
(see~\citep[p. 29]{Geer2000}).  Lemma~\ref{lem:covnum} implies that\footnote{In
  fact, $c_1=(c'+1/4)\sqrt{\log (4c'+1)} -
  \frac{\sqrt{\pi}}{8}\mbox{erfi}(\sqrt{4c'+1} )$, where $\mbox{erfi}$ is the
  imaginary error function. Despite the adjective ``imaginary'', $c_1$ is
  always real.}
\begin{align*}
  \label{eq:12}
  J(\tau) \leq 2\tau \int_0^1{dv \sqrt{\log(1+4v/c')}} \leq 2c_1\tau.
\end{align*}
Finally, we can prove an important corollary to Lemma~\ref{lem:2}. The proof
makes use of the entropy integral as well as the peeling device, and it follows
from Lemma 3.2 in~\citet{Geer2000}, so we provide only the necessary
adjustments in our proof here. However, we will need both the peeling device
and the entropy integral again in the proof of Theorem~\ref{thm:emp}.

\begin{corollary}
  [Corollary of Lemma~\ref{lem:2}]
  \label{cor}
  If $\sup_{g \in \mathcal{G}} \vnorm{g}_k \leq \tau$ and \eqref{5}
  holds for all design $n_\ell$, then for all
  \begin{equation*}
    \label{eq:15}
    \delta>\frac{\tau}{\sqrt{2km}} \max\{24c_1, 29\},
  \end{equation*}
  we have, as a consequence of Lemma~\ref{lem:2},
  \begin{equation*}
    \label{eq:16}
    \P\left( \sup_{g \in \mathcal{G}} \left|\frac{1}{k} \sum_{\ell=1}^k
        \epsilon_\ell g(n_\ell)\right| > \delta \right) \leq 4
    \exp\left\{- \frac{km\delta^2}{c_3\tau^2}\right\},
  \end{equation*}
  where $c_1$ is as above and $c_3=2304$.
\end{corollary}
\begin{proof}
  The proof is given in Lemma 3.2 in~\citet{Geer2000}. In our case, the entropy
  integral converges, so we may take $K=\infty$ in that proof. Furthermore, we
  replace equation (3.3) there with the result of Lemma~\ref{lem:2} here, and
  set $\epsilon=\delta/2$.
\end{proof}

\begin{theorem}
  \label{thm:emp}
  Suppose that $\widehat{h}$ is the solution to~\eqref{10}. Let
  \begin{equation}
    \delta>\frac{4}{\sqrt{2mk}}\max\{24c_1,29\}.\label{eq:17}
 \end{equation} Then,
  \begin{equation*}
    \label{eq:4}
    \P(\vnorm{\Phi_{\widehat{h}}- \Phi_{h^*}}_k > \delta) \leq
    13\exp\left\{-\frac{mk\delta^2}{16c_3}\right\}. 
  \end{equation*}
\end{theorem}
\begin{proof}
  First, note that $\vnorm{\Phi_{\widehat{h}} - \Phi_{h^*}}_k^2 \leq
  2(w,\Phi_{\widehat{h}} - \Phi_{h^*})_k$.  Then we use peeling and the lemmas
  above:
  \begin{align*}
    \label{eq:8b}
    \lefteqn{\P(\vnorm{\Phi_{\widehat{h}} - \Phi_{h^*}}_k > \delta)}\nonumber\\ & \leq \sum_{s=0}^\infty
    \P\left( \sup_{\Phi_h \in \mathcal{G}(2^{s+1}\delta)} (w, \Phi_h-\Phi_{h^*})_k >
      2^{2s-1}\delta^2\right)\\
    &= \sum_{s=0}^\infty \P_s.
  \end{align*}
  Now $\forall s \geq 0$, $$2^{2s-1}\delta^2 >
  \frac{2^{s+1}\delta}{\sqrt{2km}}\max\{24c_1,29\}$$ by (\ref{eq:17}),
  therefore, we can apply Corollary~\ref{cor} to each $\P_s$. This gives
  \begin{align*}
    \sum_{s=0}^\infty \P_s & \leq \sum_{s=0}^\infty 4 \exp \left\{ -
      \frac{mk2^{4s-2}\delta^4}{c_3 2^{2s+2}\delta^2} \right\}\\
    & = \sum_{s=0}^\infty 4 \exp \left\{ - \frac{mk2^{2s-4}\delta^2}{c_3}
    \right\}\\
    & = 4\exp\left\{-\frac{mk\delta^2}{16c_3}\right\} +
    4\exp\left\{-\frac{mk\delta^2}{4c_3}\right\} + \sum_{s=0}^\infty
    4\exp\left\{-\frac{2mk\delta^2 2^s}{c_3}\right\}\\ 
    & = 4\exp\left\{-\frac{mk\delta^2}{16c_3}\right\} + 4 \exp\left\{
      -\frac{mk\delta^2}{4c_3}\right\}  
    + 4\left(1- 
      \exp\left\{-\frac{2mk\delta^2}{c_3}\right\}\right)^{-1}
    \exp\left\{-\frac{2mk\delta^2}{c_3}\right\}.  
  \end{align*}
  Then, by condition (\ref{eq:17}), we have that
  $$4\left(1- \exp\left\{-\frac{2mk\delta^2}{c_3}\right\}\right)^{-1}<5$$ and the first exponential is the largest
  so we have the result.
\end{proof}
Finally, we can use the Lipschitz behavior of the function $\Phi_h$, combined
with the bound $h^*<M$ to derive our main result.
\begin{proof}[Proof of Theorem~\ref{thm:main}]
  The function $\Phi_h(n)$ is well behaved. In particular, we have that for some $c(n,M)$,
  \begin{equation*}
    \label{eq:2b}
    c(n,M)|h-h'| \leq |\Phi_h(n) - \Phi_{h'}(n)|
  \end{equation*}
  for all $h,h'<M$ and every $n$.  This is easily verified, though it is
  necessary to calculate $c(n,M)$ numerically. Therefore,
  \begin{equation*}
    \label{eq:3b}
    \frac{|h-h'|}{\sqrt{k}}\sqrt{\sum_{\ell=1}^k
          c^2(n_\ell,M)} \leq \frac{1}{\sqrt{k}}
      \sqrt{\sum_{\ell=1}^k
        \left(\Phi_h(n_\ell)-\Phi_{h'}(n_\ell)\right)^2}.
    \end{equation*}
 So setting $c_2=\frac{1}{k}\sum_{\ell=1}^k c^2(n_\ell,M)$ and applying
 Theorem~\ref{thm:emp} gives the result.
\end{proof}

\begin{proof}[Proof of Theorem~\ref{thm:bound}]
  Define $A=\{\sup_{f\in\F}|R_n(f) - \widehat{R}_n(f)|>\rho\}$ and $B=\{h^* <
  \widehat{h}+\delta\}$, then we are interested in controlling $\P(A)$. By the
  law of total probability, we have\footnote{Technically, the data
    come from the distribution $\mu$ while Algorithm~\ref{alg:gen}
    uses some other distribution, say $\nu$, to generate simulated data. Therefore,
    the probability statement in this theorem is with respect to the
    product measure $\mu\times\nu$. For the result to hold, we must
    have that $\mu$ and $\nu$ are measures over the same probability space
    $\mathcal{Y}\times\mathcal{X}$ and that the real and simulated
    data are statistically independent.}
  \begin{align*}
    \P(A) &= \P(A\given B)\P(B) + \P(A\given B^c)\P(B^c)\\
    &\leq \P(A\given B)\P(B) + \P(B^c)\\
    &=4GF(\hat{h}+\delta,2n) \exp\{-n\rho^2\}(1-\varphi) + \varphi
 \end{align*}
\end{proof}

\section{Discussion}
\label{sec:discussion}

In this paper, we showed how to derive generalization error bounds from the
estimated rather than actual VC dimension of a function class $\F$. Our method
uses the simulation procedure proposed by~\citet{VapnikLevin1994} for the
estimates.  Empirical process theory for nonparametric least squares regression
shows that these estimates $\hat{h}$ concentrate around the truth $h^*$ with
high probability. The resulting bounds can be used for model selection as well
as to characterize the finite-sample predictive ability of the model $\hat{f}$
chosen through empirical risk minimization.

The algorithm outlined here is not the only way to estimate VC dimension.
\citet{ShaoCherkassky2000} modify Algorithm~\ref{alg:gen} in light of ideas
from experimental design, varying the number of replications $m$ with the
design point $n_\ell$, and show that this improves the estimates of the VC
dimension. Modifying our empirical process techniques to use this improved
estimator would be desirable, but the extension is nontrivial.

As mentioned in the introduction, there are many other methods for measuring
the richness of a model class. Rademacher complexity in expectation is
difficult or impossible to calculate, but it has an obvious empirical
counterpart for which concentration results already exist thereby allowing for
tight data-based generalization error bounds. However, Rademacher complexity
cannot be used with unbounded loss functions. VC dimension, while discussed
here in the context of classification, generalizes to regression problems with
unbounded loss as long as appropriate moment conditions are satisfied. Hence,
our technique will apply in these settings as well.  Indeed, since VC dimension
is a property of the class of prediction functions and not the data-generating
process, and finite VC dimension has recently
\citep{Adams-Nobel-VC-classes-under-ergodic} been shown to characterize
learning from ergodic sources, it may be possible to use our procedure as part
of an algorithm for bounding prediction risk on dependent data.

\bibliography{vcestim}
\end{document}